%% file: main-2023.tex
\documentclass[letterpaper]{article} 
\usepackage{aaai23}  
\usepackage{times}  
\usepackage{helvet}  
\usepackage{courier}  
\usepackage[hyphens]{url}  
\usepackage{graphicx} 
\usepackage{natbib}  
\usepackage{caption} 
\input{notation.tex}

\title{\fw{}: Comprehensive Secure Machine Learning Framework}
\author{
    Ali Burak Ünal \textsuperscript{\rm 1,2},
    Nico Pfeifer \textsuperscript{\rm 1},
    Mete Akgün \textsuperscript{\rm 1,2}
}
\affiliations{
    \textsuperscript{\rm 1} Medical Data Privacy and Privacy-preserving ML, Dept. of Computer Science, University of Tuebingen, Tübingen, Germany\\
    \textsuperscript{\rm 2} Methods in Medical Informatics, Dept. of Computer Science, University of Tuebingen, Tübingen, Germany\\

    \{ali-burak.uenal,nico.pfeifer, mete.akguen\}@uni-tuebingen.de
}

\begin{document}

\maketitle

\begin{abstract}
Since machine learning algorithms have proven their success in many different applications, there is also a big interest in privacy preserving machine learning methods for building models on sensitive data. Moreover, the increase in the number of data sources and the high computational power required by those algorithms force individuals to outsource the training and/or the inference of a machine learning model to the clouds providing such services. To address this, we propose a secure three party computation framework, \fw{}, offering privacy preserving building blocks to enable complex operations privately. In addition to the adapted and common operations like addition and multiplication, it offers multiplexer, most significant bit and modulus conversion. The first two are novel in terms of methodology and the last one is novel in terms of both functionality and methodology. \fw{} also has two complex novel methods, which are the exact exponential of a public base raised to the power of a secret value and the inverse square root of a secret Gram matrix. We use \fw{} to realize the private inference on pre-trained recurrent kernel networks, which require more complex operations than most other deep neural networks, on the structural classification of proteins as the first study ever accomplishing the privacy preserving inference on recurrent kernel networks. In addition to the successful private computation of basic building blocks, the results demonstrate that we perform the exact and fully private exponential computation, which is done by approximation in the literature so far. Moreover, they also show that we compute the exact inverse square root of a secret Gram matrix up to a certain privacy level, which has not been addressed in the literature at all. We also analyze the scalability of \fw{} to various settings on a synthetic dataset. The framework shows a great promise to make other machine learning algorithms as well as further computations privately computable by the building blocks of the framework.
\end{abstract}

\section{Introduction}
In recent years, machine learning algorithms, especially deep learning algorithms, have become more sophisticated and are now more data-driven to achieve better performance. At the same time, the number of sources generating data with sensitive information and the amount of that data have increased dramatically. By its very nature, the privacy of this data must be protected during collaborative model training and/or inference. Likewise, the privacy of such a model must be kept private when the owner of the trained model deploys it as a prediction service.


To preserve the privacy of both the data and the model, there are several privacy techniques utilized in the literature. One of these techniques is differential privacy (DiP) introduced by \citet{dwork2014algorithmic}. In DiP, the input data, the model parameters, and/or the output are perturbed to protect the privacy by adding noise in a certain budget that adjusts the privacy level \cite{abadi2016deep,chen2020rnn}. By nature, noise addition sacrifices the performance of the model and exactness of the result to some extent. A cryptographic technique called homomorphic encryption (HE) addresses the missing points of DiP in addition to the basic privacy requirements. HE protects the privacy of the data and/or the model by encrypting the data and/or the parameters of the model with different key schemes. Thanks to this mechanism, various operations such as addition and multiplication can be performed in the encrypted domain. There are several attempts to implement machine learning algorithms using HE \cite{bakshi2020cryptornn,hesamifard2017cryptodl,gilad2016cryptonets,lu2021pegasus}. However, its huge runtime and the limited number of practical operations that can be realized with HE makes these computations impractical. Secure multi-party computation (MPC), on the other hand, addresses the missing points of HE and the fundamental basic privacy requirements. The data and/or the model parameters are shared among several parties in such a way that none of the parties can learn about the data and/or the model parameters on their own. Then, these parties perform the desired computation privately. To address various machine learning algorithms, there are several MPC frameworks in the literature \cite{knott2021crypten,wagh2019securenn,mohassel2017secureml,damgaard2012multiparty,rathee2020cryptflow2,wagh2020falcon,patra2021aby2} some of which also utilizes HE \cite{mishra2020delphi,huang2022cheetah}. Although they have some efficient and secure basic functions for several types of deep neural networks, from which we have also benefited, they lack some complex functions completely such as the inverse square root of a Gram matrix or the exact computation of some functions such as the exponential to realize the exact inference on some deep neural networks like recurrent kernel networks \cite{chen2019recurrent}.

In this study, we introduce a general purpose efficient secure multi-party computation framework, \fw{}, based on $3$ computing parties. To the best of our knowledge, the novelty of \fw{} over the other frameworks stems the privacy preserving exact computation of the exponential of a known base raised to the power of a secret shared value, the computation of the inverse square root of a secret shared Gram matrix with a high probability, a new privacy preserving modulus conversion from $2^{\numbit - 1}$ to $2^{\numbit}$, efficient private computation of the most significant bit and novel methodology of performing multiplexer in a privacy preserving way. By using \fw{}, we perform privacy preserving inference on a pre-trained recurrent kernel network, which has more complex operations than most of the other deep neural networks and cannot be realized using any other framework in the literature. As a summary, \fw{}
\begin{itemize}
    \item provides novel and/or efficient private computations for the most significant bit, multiplexer and modulus conversion to address the functions used in machine learning algorithms.
    \item introduces the exact privacy preserving computation of the exponential of a known base raised to the power of a secret shared value for the first time in the literature.
    \item enables the privacy preserving inverse square root of a secret shared Gram matrix with a high probability as the first attempt in the literature.
    \item realizes an efficient privacy preserving inference on recurrent kernel networks for the first time.
\end{itemize}



\section{Preliminaries}


\subsection{Security Model}
To prove the security of \fw{}, we use the honest-but-curious security model, in other words we analyze the security of the framework when there is a semi-honest adversary who corrupts a party and follows the protocols honestly, but at the same time tries to infer information during the execution of these protocols. In our analyses of the building blocks of \fw{}, we consider a scenario where a semi-honest adversary corrupts either one of the proxies or the helper party and tries to infer the input and/or output values of the executed function. We take the same approach to analyze the security of privacy preserving RKNs and focus on if the semi-honest adversary corrupting a party can break the privacy of the test sample of the data owner and/or the model parameters of the model owner during the prediction of a test sample using the outsourced model.

\subsection{Square-and-multiply Algorithm for Exponential} \label{sec:sq_and_mul}
One of the approaches to computing the exponential of a base raised to the power of a value that can be represented in a binary form is the square-and-multiply algorithm. Let $b$ be the base and $a$ be the power, whose binary representation is $\langle a \rangle$ of length \numbit. Then, one can compute $b^a$ as follows: 
\begin{equation}
    b^a = \prod_{i=0}^{n-1} b^{\langle a \rangle_i \cdot 2^i}    
\end{equation}
where $\langle a \rangle_i$ represents the bit value of $a$ at position $i$, assuming that the indexing starts from the least significant bit, that is $\langle a \rangle_0$ corresponds to the least significant bit and $\langle a \rangle_{n-1}$ corresponds to the most significant bit.

\subsection{Recurrent Kernel Networks}
\citet{chen2019recurrent} gave a kernel perspective of RNNs by showing that the computation of the specific construction of RNNs, which they call recurrent kernel networks (RKNs), mimics the substring kernel allowing mismatches and the local alignment kernel, which are widely used on sequence data \cite{el2008predicting,nojoomi2017weighted,leslie2001spectrum}. RKNs use a well-designed kernel formulation and enable the parameters of the model to be optimized via backpropagation. As a result, they outperform the traditional substring kernel and the local alignment kernel as well as the LSTMs \cite{hochreiter2007fast}. Such performance promises to have a large impact on the results of tasks with small to medium size data.

In RKNs, small motifs, called \textit{anchor points}, are used as templates to measure similarities among sequences and each character of an anchor point is encoded via an encoding similar to one-hot encoding. Let there be $\numanc$ anchor points of length $\kmer$, each of whose characters is encoded using a vector of size $\onehotdim$, and a sequence $x$ of length $\seq$, whose characters are encoded using one-hot encoded vectors of length $\onehotdim$. For the $t$-th character of the input sequence, RKNs first compute the similarity of this character to each character of all anchor points by the following:
\begin{equation}
    K(x_t, z_j^i) = e^{\alpha(\langle x_t, z_j^i \rangle - 1)}
    \label{eq:sim}
\end{equation}
where $\alpha$ is the similarity measure parameter, $x_t$ is the $\onehotdim$-dimensional vector representing the $t$-th character of the sequence $x$ and $z_j^i$ is the $\onehotdim$-dimensional vector representing the $i$-th character of the $j$-th anchor point. Once the similarity of the $t$-th character of the sequence to the $j$-th character of all $\numanc$ anchor points is computed and represented as a vector $b_j[t]$, the computation continues as follows:
\begin{equation}
    c_j[t] = \lambda c_j[t-1] + c_{j-1}[t-1] \otimes b_j[t]
    \label{eq:iter}
\end{equation}
where $c_j[t]$ is the initial mapping of the sequence up to its $t$-th character into a $\numanc$-dimensional vector based on anchor points of length $j$ for $j \in \{1,\ldots,\kmer\}$ and $t \in \{1,\ldots,\seq\}$, $\lambda$ is a scalar value to downgrade the effect of the previous time points and $\otimes$ is the Hadamard product. As the base of this recursive equation, $c_0[t]$ is a vector of $1$s and $c_j [0]$ is a vector of $0$s if $j \neq 0$.

To obtain the final mapping of the sequence, RKN multiplies $c_j[\seq]$ by the inverse square root matrix of the Gram matrix of the anchor points up to their $j$-th characters. Afterwards, in the substring kernel allowing mismatches, they use a linear layer, which is a dot product between $c_{\kmer}[\seq]$ and the weight vector of the classifier $\clas$, to obtain the prediction of the input sequence. Figure \ref{fig:rkn_circuits} depicts these computations.

\begin{figure}[t]
    \centering
    \begin{subfigure}[b]{0.28\textwidth}
        \hspace{-4mm}
        \includegraphics[width=\textwidth]{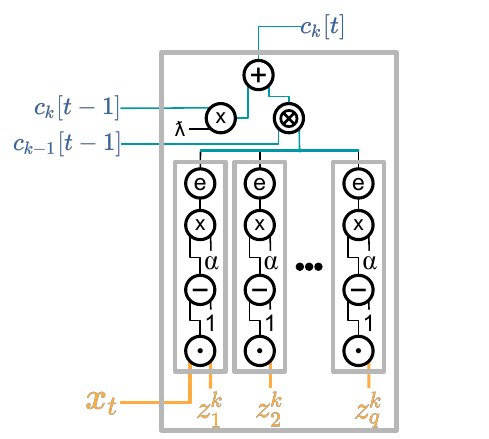}
        \caption{}
        \label{fig:rkn_single_neuron}
    \end{subfigure}
    \hspace{-8mm}
    \begin{subfigure}[b]{0.18\textwidth}
        \centering
        \includegraphics[width=\textwidth]{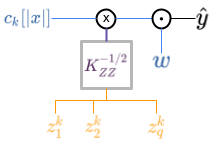}
        \caption{}
        \label{fig:rkn_linear_classifier}
    \end{subfigure}
    \caption{ The arithmetic circuit of the computation in \textbf{(a)} a single neuron of RKN at position $t$ and $k$-mer level $k$ and \textbf{(b)} the linear classifier layer of RKN after the last position of the input sequence $x$. Black lines are a scalar value. Orange lines are $\onehotdim$-dimensional vectors and blue lines are $\numanc$-dimensional vectors. Purple ones are $(\numanc \times \numanc)$-dimensional matrix}
    \label{fig:rkn_circuits}
\end{figure}

\section{Framework}
\fw{} is based on three computing parties and 2-out-of-2 additive secret sharing. Two of these parties, $\party_0$ and $\party_1$, are called \textit{proxy} and the external entities such as the model owner and the data sources interact with them. The third one, $\party_2$, is the \textit{helper} party helping the proxies compute the desired function without breaking the privacy. It provides the proxies with the shares of the purposefully designed values. It also performs calculations on the data masked by the proxies in the real number domain and returns the share of the results of these calculations. In the rest of this section, we describe the number format we use to represent values. We then introduce the methods of \fw{}. 


\subsection{Notations}
To ensure the privacy in \fw{}, we use 2-out-of-2 additive secret sharing over three different rings $\mathbb{Z}_{\ringA}$, $\mathbb{Z}_{\ringB}$ and $\mathbb{Z}_{\ringC}$ where $\ringA = 2^{\numbit}$, $\ringB = 2^{\numbit-1}$, $\ringC=67$ and $\numbit=64$. We denote two shares of $x$ over $\mathbb{Z}_{\ringA}$, $\mathbb{Z}_{\ringB}$ and $\mathbb{Z}_{\ringC}$ by ($\langle x\rangle_0$, $\langle x\rangle_1$), ($\langle x\rangle_0^\ringB$, $\langle x\rangle_1^\ringB$) and ($\langle x\rangle_0^\ringC$, $\langle x\rangle_1^\ringC$), respectively. If a value $x$ is shared over the ring $\mathbb{Z}_\ringC$, every bit of $x$ is additively shared in $\mathbb{Z}_\ringC$. This means $x$ is shared as a vector of $\numbit$ shares where each share takes a value between $0$ and $(\ringC - 1)$. We also use boolean sharing of a single bit denoted by ($\langle x\rangle_0^\ringBool$, $\langle x\rangle_1^\ringBool$).

\subsection{Number Format}
Since machine learning algorithms can require both positive and negative real numbers, \fw{} should be able to work with these numbers. To achieve this, we follow the number format proposed by \citet{mohassel2017secureml} in SecureML. In this number format, the most significant bit indicates the sign. $0$ and $1$ indicate positive and negative values, respectively. Moreover, a certain number of least significant bits of this representation are for the fractional part of the value and the rest expresses the integer part.

\subsection{Addition (ADD)}
In \fw{}, the proxies add the shares of two secret shared values they have to obtain the share of the addition of these values without any communication and privacy leakage.

\subsection{Multiplication (MUL)}
The multiplication operation, which we adapted from SecureML, uses the pre-computed multiplication triples \cite{DBLP:conf/crypto/Beaver91a} and requires \textit{truncation} because of the special number format. For more details, refer to the Supplement and see \cite{mohassel2017secureml,wagh2019securenn}.



\IncMargin{0.4em}
\begin{algorithm}[t]
\footnotesize
\DontPrintSemicolon
\SetKwInOut{Input}{input}\SetKwInOut{Output}{output}
\SetKwFunction{algo}{$\mathsf{MOC}$}
\SetKwProg{myalg}{Algorithm}{}{}
\myalg{\algo{}}{
\Input{$\party_0$ and $\party_1$ hold $\langle x\rangle_0^{\ringB}$ and $\langle x\rangle_1^{\ringB}$, respectively}   
\Output{$\party_0$ and $\party_1$ get $\langle x\rangle_0$ and $\langle x\rangle_1$, respectively}
$\party_0$ and $\party_1$ hold a common random bit $u^\prime$\;
$\party_2$ picks a random numbers $r \in \mathbb{Z}_{\ringB}$ and generates $\langle r\rangle_0^{\ringB}$, $\langle r\rangle_1^{\ringB}$, 
$\{\langle r[j]\rangle_{0}^{\ringC}\}_{j \in[\numbit]}$ and $\{\langle r[j]\rangle_{1}^{\ringC}\}_{j \in[\numbit]}$.\;
$\party_2$ computes $w = \text{isWrap}(\langle r\rangle_0^{\ringB}, \langle r\rangle_1^{\ringB},\ringB)$ and divides $w$ into two boolean shares $w_0^B$ and $w_1^B$\;
$\party_2$ sends $\langle r\rangle_i^{\ringB}$, $\{\langle r[j]\rangle_{i}^{\ringC}\}_{j \in[\numbit]}$ and $w_i^B$ to $\party_i$, for each $i \in \{0, 1\}$\;
For each $i \in \{0, 1\}$, $\party_i$ executes Steps $7$-$8$\;
$\langle y\rangle_i^{\ringB} = \langle x\rangle_i^{\ringB}+\langle r\rangle_i^{\ringB}$\;
$\party_i$ reconstructs $y$ by exchanging shares with $\party_{1-i}$\;
$u_i^B = \mathsf{PC}(\{\langle r[j]\rangle_{i}^{\ringC}\}_{j \in[\numbit]},y,u^\prime)$\;
$\party_0$ computes $u_i^B = u_i^B \oplus u^\prime$\;
For each $i \in \{0, 1\}$, $\party_i$ computes $c_i^B = w_i^B\oplus u_i^B$\;
$\party_0$ computes $\langle y\rangle_0 = \langle y\rangle_0^{\ringB} + \text{isWrap}(\langle y\rangle_0^{\ringB}, \langle y\rangle_1^{\ringB}, \ringB) \cdot \ringB$\;
$\party_1$ sets  $\langle y\rangle_1 = \langle y\rangle_1^{\ringB}$\;
For each $i \in \{0, 1\}$, $\party_i$ computes
$\langle x\rangle_i = \langle y\rangle_i - (\langle r\rangle_i^{\ringB} + c_i^B \cdot \ringB)$\;
}
\caption{Modulus Conversion ($\mathsf{MOC}$)}
\label{alg:mc}
\end{algorithm}

\subsection{Modulus Conversion (MOC)} 
\fw{} offers the functionality $\mathsf{MOC}$ described in Algorithm \ref{alg:mc}. It converts shares over $\mathbb{Z}_{\ringB}$ to fresh shares over $\mathbb{Z}_{\ringA}$ where $\ringA=2\ringB$. Even though other frameworks in the literature have functions with a similar name, none of those performs this specific modulus conversion. To the best of our knowledge, \fw{} is the first framework offering this specific modulus conversion. Assuming that $\party_0$ and $\party_1$ have the shares $\langle x\rangle_0^\ringB$ and $\langle x\rangle_1^\ringB$, respectively, the first step for $\party_0$ and $\party_1$ is to mask their shares by using the shares of the random value $r \in \mathbb{Z}_\ringB$ sent by $\party_2$. Afterwards, they reconstruct  $(x+r) \in \mathbb{Z}_\ringB$ by first computing $\langle y \rangle_i^\ringB = \langle x \rangle_i^\ringB + \langle r \rangle_i^\ringB$ for $i \in \{0,1\}$ and then sending these values to each other. Along with the shares of $r \in \mathbb{Z}_\ringB$, $\party_2$ also sends the information in boolean shares telling whether the summation of the shares of $r$ wraps so that $\party_0$ and $\party_1$ can convert $r$ from the ring $\mathbb{Z}_\ringB$ to the ring $\mathbb{Z}_\ringA$. Once they reconstruct $y \in \mathbb{Z}_\ringB$, $\party_0$ and $\party_1$ can change the ring of $y$ to $\mathbb{Z}_\ringA$ by adding $\ringB$ to one of the shares of $y$ if $\langle y \rangle_0^\ringB + \langle y \rangle_1^\ringB$ wraps. After conversion, the important detail regarding $y \in \mathbb{Z}_\ringA$ is to fix the value of it. In case $(x+r) \in \mathbb{Z}_\ringB$ wraps, which we identify by using the private compare ($\mathsf{PC}$) method \cite{wagh2019securenn}, depending on the boolean share of the outcome of $\mathsf{PC}$, $\party_0$ or $\party_1$ or both add $\ringB$ to their shares. If both add, this means that there is no addition to the value of $y \in \mathbb{Z}_\ringA$. At the end, $\party_i$ subtracts $r_i \in \mathbb{Z}_\ringA$ from $y_i \in \mathbb{Z}_\ringA$ and obtains $x_i \in \mathbb{Z}_\ringA$ for $i \in \{0,1\}$.

\subsection{Most Significant Bit (MSB)}

One of the biggest improvements that \fw{} introduces is in the private determination of the most significant bit of a secret shared value $x$ via $\mathsf{MSB}$. We deftly integrated $\mathsf{MOC}$ and PC into $\mathsf{MSB}$ to reduce communication rounds. Given the shares of $x$, $\mathsf{MSB}$ first extracts the least significant $(\numbit - 1)$-bits via $mod\ \ringB$. Then, it converts the ring of this value from $\mathbb{Z}_\ringB$ to $\mathbb{Z}_\ringA$. $\mathsf{MSB}$ then subtracts it from $x$, which results in either $0$ or $\ringB$ in $\mathbb{Z}_\ringA$. Finally, it secretly maps this result to $0$ or $1$ and obtains the most significant bit of $x$ privately.

\begin{algorithm}[!htb]
    \footnotesize
    \DontPrintSemicolon
    \SetKwInOut{Input}{input}\SetKwInOut{Output}{output}
    \SetKwFunction{algo}{$\mathsf{MSB}$}
    \SetKwProg{myalg}{Algorithm}{}{}
    \myalg{\algo{}}{
        \Input{$\party_0$ and $\party_1$ hold $\langle x\rangle_0$ and $\langle x\rangle_1$, respectively}
        \Output{$\party_0$ and $\party_1$ get $\langle z\rangle_0$ and $\langle z\rangle_1$, respectively, where $z$ is equal to $0$ if the most significant bit of $x$ is $0$ and $1$ otherwise.}
        $\party_0$ and $\party_1$ hold a common random bit $f$ and $g$. $\party_0$ and $\party_1$ additionally hold $\ell$ common random values $s_{j} \in \mathbb{Z}_{\ringC}^{*}$ for all $j \in[\ell]$, a random permutation $\pi$ for $\ell$ elements and $\ell$ common random values $u_{j} \in \mathbb{Z}_{\ringC}^{*}$.\;
        $\party_2$ picks a random number $r \in \ringB$ and generates $\langle r\rangle_0^{\ringB}$, $\langle r\rangle_1^{\ringB}$, 
        $\{\langle r[j]\rangle_{0}^{\ringC}\}_{j \in[\ell]}$ and $\{\langle r[j]\rangle_{1}^{\ringC}\}_{j \in[\ell]}$.\;
        $\party_2$ computes $w = \text{isWrap}(\langle r\rangle_0^{\ringB}, \langle r\rangle_1^{\ringB},\ringB)$\;
        $\party_2$ sends $\langle r\rangle_i^{\ringB}$ and $\{\langle r[j]\rangle_{i}^{\ringC}\}_{j \in[\ell]}$ to $\party_i$, for each $i \in \{0, 1\}$\;
        For each $i \in \{0, 1\}$, $\party_i$ executes Steps $7$-$26$\;
        $\langle d\rangle_i^{\ringB} = \langle x\rangle_i \mod \ringB$\;
        $\langle y\rangle_i^{\ringB} = \langle d\rangle_i^{\ringB}+\langle r\rangle_i^{\ringB}$\;
        $y=\mathsf{Reconst(\langle y\rangle_i^{\ringB})}$\;
        $\langle y\rangle_i = \langle y\rangle_i^{\ringB} + i\cdot\text{isWrap}(\langle y\rangle_0^{\ringB}, \langle y\rangle_1^{\ringB}, \ringB) \cdot \ringB$\;
        $\langle a[0]\rangle_i = if\ringB - \langle x\rangle_i + \langle y\rangle_i - \langle r\rangle_i^{\ringB}$\; 
        $\langle a[1]\rangle_i = i(1-f)\ringB - \langle x\rangle_i + \langle y\rangle_i - \langle r\rangle_i^{\ringB}$\; 
        
        Let $t=y+1 \bmod 2^{\ell}$\;
        \For{$j = \ell;\ j > 0;\ j = j - 1$}{
          \uIf{$g=0$}{
            $\left\langle w_{j}\right\rangle_{i}^{\ringC}=\langle r[j]\rangle_{i}^{\ringC}+i y[j]-2 y[j]\langle r[j]\rangle_{i}^{\ringC}$\;
            $\left\langle c_{j}\right\rangle_{i}^{\ringC}=i y[j]-\langle r[j]\rangle_{i}^{\ringC}+j+\sum_{k=j+1}^{\ell}\left\langle w_{k}\right\rangle_{i}^{\ringC}$\;
          }
          \uElseIf{$g=1$ AND $r \neq 2^{\ell}-1$}{
            $\left\langle w_{j}\right\rangle_{i}^{\ringC}=\langle r[j]\rangle_{i}^{\ringC}+i t[j]-2 t[j]\langle r[j]\rangle_{i}^{\ringC}$\;
            $\left\langle c_{j}\right\rangle_{i}^{\ringC}=-i t[j]+\langle r[j]\rangle_{i}^{\ringC}+i+\sum_{k=j+1}^{\ell}\left\langle w_{k}\right\rangle_{i}^{\ringC}$;\
          }
          \Else{
            \uIf{$i \neq 1$}{
                $\left\langle c_{j}\right\rangle_{i}^{\ringC}=(1-i)\left(u_{j}+1\right)-i u_{j}$\;
            }
            \Else{
                $\left\langle c_{j}\right\rangle_{i}^{\ringC}=(-1)^{j} \cdot u_{j}$\;
            }
          }
        }
        
        Send $\left\{\left\langle b_{j}\right\rangle_{i}^{\ringC}\right\}_{j} = 
            \pi\left(\left\{ s_{j}\left\langle c_{j}\right\rangle_{i}^{\ringC} \right\}_{j}\right)$ 
            and $\langle a\rangle_i$ to $\party_2$ \;
        For all $j \in {[\ell]}$, $\party_2$ computes $d_{j} = 
            \mathsf{Reconst}(\langle d_{j}\rangle_{0}^{\ringC}, \langle d_{j}\rangle_{1}^{\ringC})$ 
            and sets $g^{\prime} = 1$ iff $\exists j \in [\ell]$ 
            such that $d_{j} = 0$. \;
        $\party_2$ reconstructs $a[j]$ where $j \in \{0, 1\}$ 
            and computes $a[j] = (a[j] - (g^{\prime} \oplus w) \ringB) / \ringB$\;
        $\party_2$ creates two fresh shares of $a[j]$ where $j \in \{0, 1\}$ and sends them to $\party_0$ and $\party_1$\;
        For each $i \in \{0, 1\}$, $\party_i$ executes Step $31$\;
        $\langle z\rangle_i = \langle a[f \oplus g]\rangle_i$\;
    }
    \caption{Most Significant Bit ($\mathsf{MSB}$)}
    \label{alg:msb}
\end{algorithm}

\subsection{Comparison of Two Secret Shared Values (CMP)}
\fw{} also provides $\mathsf{CMP}$ which privately compares two secret shared values $x$ and $y$, and outputs $0$ if $x \geq y$, $1$ otherwise in secret shared form. $\mathsf{CMP}$ utilizes $\mathsf{MSB}$ to determine the most significant bit of $x - y$. Then, it outputs $0$ if $\mathsf{MSB}$ gives $0$, $1$ otherwise.

\subsection{Multiplexer (MUX)}
\fw{} also provides the functionality $\mathsf{MUX}$ described in Algorithm \ref{alg:ss}. It performs the selection of one of two secret shared values based on the secret shared bit value using the randomized encoding (RE) of multiplication \citep{applebaum2017garbled}. In $\mathsf{MUX}$, the proxies compute Equation \ref{eq:sscom} using their shares ($\langle x\rangle_0$, $\langle y\rangle_0$, $\langle b\rangle_0$) and ($\langle x\rangle_1$, $\langle y\rangle_1$, $\langle b\rangle_1$), respectively, and obtain the fresh shares of $z=x -b(x-y)$. As shown in Equation \ref{eq:sscom}, the proxies need to multiply two values owned by different parties in the computation of $\langle b\rangle_0(\langle x\rangle_1 - \langle y\rangle_1)$ and $\langle b\rangle_1(\langle x\rangle_0 - \langle y\rangle_0$. They outsource these multiplications to the helper via the RE. To the best of our knowledge, it is the first time that the RE is used in such a functionality.
\begin{equation}
\small
\begin{split}
z ={}& x-b(x-y) \\
 ={}& \langle x\rangle_0 + \langle x\rangle_1 - \langle b\rangle_0(\langle x\rangle_0 - \langle y\rangle_0) - \langle b\rangle_1(\langle x\rangle_1 - \langle y\rangle_1) \\
 & - \langle b\rangle_0(\langle x\rangle_1 - \langle y\rangle_1)- \langle b\rangle_1(\langle x\rangle_0 - \langle y\rangle_0)
\end{split}
\label{eq:sscom}
\end{equation}

\begin{algorithm}[!ht]
\footnotesize
\DontPrintSemicolon
\SetKwInOut{Input}{input}\SetKwInOut{Output}{output}
\SetKwFunction{algo}{$\mathsf{MUX}$}
\SetKwProg{myalg}{Algorithm}{}{}
\myalg{\algo{}}{
\Input{$\party_0$ and $\party_1$ hold ($\langle x\rangle_0$, $\langle y\rangle_0$, $\langle b\rangle_0$) and ($\langle x\rangle_1$, $\langle y\rangle_1$, $\langle b\rangle_1$), respectively.}   
\Output{$\party_0$ and $\party_1$ get $\langle z\rangle_0$ and $\langle z\rangle_1$, respectively, where $z=x-b(x-y)$.}
$\party_0$ and $\party_1$ hold four common random values $r_i$ where $i \in \{0,1,2,3\}$\;
$\party_0$ computes $M_1=\langle x\rangle_0-\langle b\rangle_0(\langle x\rangle_0-\langle y\rangle_0)+r_1\langle b\rangle_0 + r_2(\langle x\rangle_0-\langle y\rangle_0)+r_2r_3$, $M_2=\langle b\rangle_0+r_0$, $M_3=\langle x\rangle_0-\langle y\rangle_0+r_3$\;
$\party_0$ sends $M_2$ and $M_3$ to $\party_2$\;
$\party_1$ computes $M_4=\langle x\rangle_1-\langle b\rangle_1(\langle x\rangle_1-\langle y\rangle_1) + r_0(\langle x\rangle_1-\langle y\rangle_1)+r_0r_1 + r_3\langle b\rangle_1$, $M_5=(\langle x\rangle_1-\langle y\rangle_1)+r_1$, $M_6=\langle b\rangle_1+r_2$\;
$\party_1$ sends $M_5$ and $M_6$ to $\party_2$\;
$\party_2$ computes $M_2M_5+M_3M_6$ $=z$\;
$\party_2$ divides $z$ into two shares $(\langle z\rangle_0+\langle z\rangle_1)$ and sends $\langle z\rangle_0$ and $\langle z\rangle_1$ to $\party_0$ and $\party_1$, respectively\;
$\party_0$ computes $\langle z\rangle_0=M_1 - \langle z\rangle_0$\;
$\party_1$ computes $\langle z\rangle_1=M_4-\langle z\rangle_1$\;
}
\caption{Multiplexer ($\mathsf{MUX}$)}
\label{alg:ss}
\end{algorithm}


\subsection{Dot Product (DP)}
Since the dot product between two vectors is a widely used operation in machine learning algorithms, \fw{} provides $\mathsf{DP}$. It essentially uses $\mathsf{MUL}$ to compute the element-wise multiplication of the vectors and $\mathsf{ADD}$ to sum the elements of the resulting vector to obtain the dot product result.

\subsection{Exponential Computation (EXP)}
One of the novel methods of \fw{} is functionality $\mathsf{EXP}$ given in \ref{alg:exp}. It computes the exponential of a publicly known base raised to the power of a given secret shared value. For this purpose, we have been inspired by the square-and-multiply algorithm. To the best of our knowledge, \fw{} is the first study that extends the core idea of the square-and-multiply algorithm to cover not only the positive numbers, but also the negative numbers as well as their decimal parts in a multi-party scenario and performs the exact exponential computation. As an overview, the proxies first obtain the most significant bit of the secret shared power and use this to select the set containing either the power itself and the contribution of each bit of a positive power or the absolute of the power and the contribution of each bit of a negative power. Then, the proxies determine the value of each bit of the power in a secret sharing form and use them to select between the previously selected contributions of the bits and a vector of $1$s. The last step is to multiply these selected contribution of the bits of the power to the exponential in a binary-tree-like-structure. In total, $\mathsf{EXP}$ requires two $\mathsf{MSB}$, two $\mathsf{MUX}$ and $log_2(\numbit)$-many $\mathsf{MUL}$ calls.

\begin{algorithm}[!ht]
\footnotesize
\DontPrintSemicolon
\SetKwInOut{Input}{input}\SetKwInOut{Output}{output}
\SetKwFunction{algo}{$\mathsf{EXP}$}
\SetKwProg{myalg}{Algorithm}{}{}
\myalg{\algo{}}{
\Input{$\party_0$ and $\party_1$ hold $\langle x \rangle_0$ and $\langle x \rangle_1$, respectively, and publicly known base $b$}   
\Output{$\party_0$ and $\party_1$ get $\langle z\rangle_0$ and $\langle z\rangle_1$, respectively}

For $i \in \{0,1\}$, $\party_i$ executes Steps 3-15 with the help of $\party_2$ \;

$\langle s\rangle_i = \mathsf{MSB}(\langle x\rangle_i)$

$\langle |x|\rangle_i = 0 - \langle x\rangle_i$

\For{$j = \numbit;\ j > 0;\ j = j - 1$}{
   $\langle cP[\numbit - j] \rangle_i = i * b^{2^{j - \dec}}$\;
   $\langle cN[\numbit - j] \rangle_i = i * (1 / b^{2^{j - \dec}}$)
}

$(\langle \hat{x} \rangle_i, \langle cONE \rangle_i) = \mathsf{MUX}((\langle x \rangle_i, \langle cP \rangle_i), (\langle |x| \rangle_i, \langle cN \rangle_i), \langle s\rangle_i)$

\For{$j = 0;\ j < \numbit;\ j = j + 1$}{
   $\langle \hat{x}^{\numbit - j} \rangle_i = (\langle \hat{x} \rangle_i << j)$
}

$\langle M \rangle_i = \mathsf{MSB}((\langle \hat{x}^{64} \rangle_i, \langle \hat{x}^{63} \rangle_i, \ldots, \langle \hat{x}^{1} \rangle_i))$

$\langle cACT \rangle_i = \mathsf{MUX}(\langle cONE \rangle_i, \langle cZERO \rangle_i, \langle M \rangle_i)$

$\langle z \rangle_i = \langle cACT \rangle_i$ \;
\For{$j = 0;\ j < log_2(\numbit);\ j = j + 1$}{
   $\langle z \rangle_i = \mathsf{MUL}(\langle z[0:len(z)/2] \rangle_i, \langle z[len(z)/2:len(z)] \rangle_i)$
}
}
\captionsetup{width=\linewidth}
\caption{Exponential computation ($\mathsf{EXP}$)}
\label{alg:exp}
\end{algorithm}

\subsection{Inverse Square Root of Gram Matrix (INVSQRT)}

As a unique feature, \fw{} provides a more sophisticated and special operation, namely the inverse square root of a secret shared Gram matrix, whose pseudocode is given in Algorithm \ref{alg:invsqrt}. To the best of our knowledge, \fw{} is the first study in the literature offering such a functionality.
In $\mathsf{INVSQRT}$, we follow the idea of computing the reciprocal of the square root of the eigenvalues of the Gram matrix and constructing the desired matrix using the diagonal matrix of these reciprocals of the square root of the eigenvalues and the original eigenvectors. For the privacy preserving eigenvalue decomposition, we were inspired by the approach proposed by \citet{zhou2016outsourcing}, where the eigenvalue decomposition was designed as an outsourcing operation from a single party. We adapted this approach to the multi-party scenario so that the proxies with a share of Gram matrix can perform the eigenvalue decomposition without learning anything other than their share of the outcome. 

Let $i \in \{0,1\}$, $\G_i$ be the share of the Gram matrix in $\party_i$, $\M$ and $(\alp, \s)$ be a common orthogonal matrix and common scalars, respectively, known by both proxies. $\party_i$ first masks its share of the Gram matrix by computing $\G^{'}_{i} = \M (\alp \G_i + \s I) \M^T$, where $I$ is the identity matrix, and then sends it to the helper party, or $\party_2$. First, $\party_2$ reconstructs $\G^{'}$ and then performs an eigenvalue decomposition on $\G^{'}$. It obtains the masked versions of the original eigenvalues, represented as $\evl^{'}$, and eigenvectors, represented as $\evc^{'}$. After splitting $\evc^{'}$ into two secret shares, $\party_2$ generates a vector of random values $\Delta$ and a random scalar $\alpha$, and sends them to $\party_1$ along with $\evc^{'}_1$. Once $\party_1$ has received $\evc^{'}_1$, $\Delta$ and $\alpha$, it computes $\evc_1 = \M^T \evc^{'}_1$ and the vector $U = \s \Delta + \alpha$, which will act as \textit{partial unmasker}. Afterwards, $\party_1$ sends $U$ to $\party_0$. The final task of $\party_2$ is to mask the masked eigenvalues by computing $\evl'' = \evl' \odot \Delta + \alpha$ and to send it along with $\evc^{'}_0$ to $\party_0$. Once $\party_0$ receives $\evc^{'}_0$, $U$ and $\evl''$, it computes $\evc_0 = \M^T \evc^{'}_0$ and $\evl''' = (\evl'' - U) / \alp$. To obtain the shares of the reciprocal of the square root of the original eigenvalues, i.e. $\evl^{-1/2}$, $\party_0$ and $\party_1$ perform private multiplication with inputs $((\evl''')^{-1/2}, 0)$ and $(0, \Delta^{1/2})$, respectively. The final step for the proxies to obtain the share of the inverse square root of $\G$ is to reconstruct it by computing $\G_i^{-1/2} = \evc_i \cdot diag(\evl^{-1/2}) \cdot \evc_{i}^{T}$ for $i \in \{0,1\}$.

Note that $\mathsf{INVSQRT}$ is a private method with a high probability, since $\M$, $\alp$, $\s$, $\Delta$ and $\alpha$ must be in a problem specific range not to cause overflow in the masking procedures and loss of masked values' connection to the real numbers. This restriction affects the feature of complete randomness of the shares in MPC. Based on the resulting values from the computation, the parties can reduce the range of secret values, which in this case are the values in $\G$, $\evc$ and $\evl$, to a range slightly smaller than the original range. However, to perform private inference on RKNs, the inverse square root of the required Gram matrices can be outsourced along with the other parameters of the model and the randomness regarding the aforementioned values is completely preserved. The reason of including $\mathsf{INVSQRT}$ here is to initiate the effort of realizing such a computation, which is required if one wishes to train an RKN or other algorithms requiring this operation via MPC. Main contribution of this paper, i.e. the realization of private inference on RKNs via efficient and/or previously unaddressed functions such as the exact exponential, is still valid even if we use the outsourced inverse square root of the Gram matrices. We discuss further details in Supplement.

\begin{algorithm}[!htb]
\footnotesize
\DontPrintSemicolon
\SetKwInOut{Input}{input}\SetKwInOut{Output}{output}
\SetKwFunction{algo}{$\mathsf{INVSQRT}$}
\SetKwProg{myalg}{Algorithm}{}{}
\myalg{\algo{}}{
\Input{$\party_0$ and $\party_1$ hold the share of Gram matrix $\langle G \rangle_0$ and $\langle G \rangle_1$, respectively, and common random scalars $\s$ and $\alp$, and a common random matrix $\M$}   
\Output{$\party_0$ and $\party_1$ get $\langle G_{invsqrt} \rangle_0$ and $\langle G_{invsqrt} \rangle_1$, respectively}

$\party_i$ locally computes $\langle G' \rangle_i = \M (\alp \langle G \rangle_i + \s I) \M^T$ and sends $\langle G' \rangle_i$ to $\party_2$  for $i \in \{0,1 \}$ \;

$\party_2$ computes  $G' = \langle G' \rangle_0 + \langle G' \rangle_1$ to reconstruct the masked Gram matrix \;

$\party_2$ computes $(\evl', \evc') = localEIG(G')$ to locally perform eigenvalue decomposition \;

$\party_2$ splits the masked eigenvectors $\evc'$ into two shares $\langle \evc' \rangle_0$ and $\langle \evc' \rangle_1$ \;

$\party_2$ generates a vector of random values $\Delta$ and a random scalar $\alpha$, and sends $\Delta$, $\alpha$ and $\langle \evc' \rangle_1$ to $\party_1$\;

$\party_2$ locally computes $\evl'' = (\evl' \odot \Delta + \alpha)$ to mask the masked eigenvalues and sends $\evl''$ and $\langle \evc' \rangle_0$ to $\party_0$ \;

$\party_1$ computes $\langle \evc \rangle_1 = \M^T \langle \evc' \rangle_1$ and $U = \s \Delta + \alpha$ \;

$\party_1$ sends $U$ to $\party_0$ \;

$\party_0$ computes $\langle \evc \rangle_0 = \M^T \langle \evc' \rangle_0$ and $\evl''' = (\evl'' - U) / \alp$ \;

$\party_0$ and $\party_1$ call $\mathsf{MUL}$ with inputs $(((\evl''')^{-1/2}, 0)$ and $(0, \Delta^{1/2}))$, respectively, to obtain $\langle \evl \rangle_i$ \;

$\party_i$ computes $\langle G_{invsqrt} \rangle_i =  \mathsf{MUL}( \mathsf{MUL}( \langle \evc \rangle_i, diag(\langle \evl \rangle_i)), \langle \evc \rangle_i)$
}
\captionsetup{width=\linewidth}
\caption{Inverse Square Root of a Gram Matrix ($\mathsf{INVSQRT}$)}
\label{alg:invsqrt}
\end{algorithm}

\section{Privacy Preserving RKNs (\app{})}
In this section, we explain the procedure of the privacy preserving inference on RKNs (\app{}) via \fw{}.

\subsection{Outsourcing}
The first step for the private inference is to outsource the model parameters from the model owner and the test sample from the data owner to the proxies. In the outsourcing of the model parameters, which are the anchor point matrix, the linear classifier weights and the inverse square root of the Gram matrices of the anchor points if one wants to have the full randomness, the model owner splits them into two arithmetic shares and sends them to the proxies in such a way that each proxy has a single share of each parameter. To outsource the test samples, the data owner proceeds similarly after using one-hot encoding to convert a sequence into a vector of numbers. It divides this vector into two shares and sends them to the proxies.

\subsection{Private Inference}
After outsourcing the model parameters and the test sample, we use the building blocks of \fw{} to realize the private inference on a pre-trained RKN. Let $t \in \{1,\ldots,\seq\}$ be the index of the characters in the sequence. The first step is to compute the similarity of the one-hot encoded $t$-th character of the sequence to each character of each anchor point via Equation \ref{eq:sim}. Such a similarity calculation involves the dot product of two secret shared vectors, the subtraction of a plaintext scalar value from a secret shared value, the multiplication of a plaintext scalar value by a secret shared value and the exponential of a known base raised to the power of a secret shared value. The output of this process corresponds to the component $b_j[t]$ in Equation \ref{eq:iter}. Once the similarity computation is complete, the proxies proceed with the private elemenent-wise product between $b_j[t]$ and $c_{j-1}[t-1]$, which is the initial mapping of the sequence up to the $(t-1)$-th character based on the anchor points of length $(j-1)$. Then the proxies add the result of the element-wise product with the downgraded $c_j[t-1]$ with a plaintext scalar value $\lambda$. At the end of this computation, the proxies obtain the secret shared initial mapping of the sequence $c_j[t]$ up to the $t$-th character to $\numanc$-dimensional space based on each anchor point of length $j \in \{1,\ldots,\kmer$\}.

After computing the initial mapping of the sequence up to its total length, that is obtaining $c_k[s]$, the proxies either compute the inverse square root of the Gram matrices of the anchor points up to each length of the anchor points via $\mathsf{INVSQRT}$, which is $K_{Z_j Z_j}^{-1/2}$ where $Z_j$ is the anchor points up to the $j$-th character, or use the inverse square root of the Gram matrices directly if they are outsourced. To obtain the final mapping, they multiply these matrices by the corresponding initial mapping vectors of the sequence. Afterwards, the proxies perform the private dot product computation of two secret shared vectors, which are the weights of the classifier and the mapping of the sequence. In the end, they obtain the prediction of the \app{} for the given sequence in a secret shared form. These shares can then be sent back to the owner of the data enabling the reconstruction of the prediction of the \app{}.

\section{Security Analysis} \label{sec:sec_analysis}
The security analyses of the methods of \fw{} and the protocol \app{} are given in the Supplement. For the security analysis of the adapted functions from other studies, we kindly refer the reader to the corresponding papers.



\begin{table}[!htb]
    \scriptsize
    \centering
    \renewcommand{\arraystretch}{1.0}
    \begin{tabularx}{\linewidth}{ABAB}
        \toprule
         Protocol & RC & Protocol & RC \\
        \midrule
        MUL & 2 & MUX & 2 \\
        MOC & 4 & DP & 2 \\
        MSB & 4 & EXP & 24 \\
        CMP & 4 & INVSQRT & 15 \\
        \bottomrule
    \end{tabularx}
    \caption{Round complexities (RC) of \fw{}}
    \label{tab:complexity}
\end{table}

\section{Complexity Analysis of the Framework}

We analyze the communication round complexities (RC) of the methods in \fw{}. Table \ref{tab:complexity} gives RCs required by each method in \fw{} to perform the designated operation. Note that we also give the analysis of the methods that we adapted from the other studies to give the comprehensive view of \fw{}.


\section{Results}

\subsection{Dataset}
To replicate the experiments of the RKN and show the correctness of \app{}, we utilized the same dataset as \citet{chen2019recurrent}, namely Structural Classification of Proteins (SCOP) version 1.67 \cite{murzin1995scop}. It contains $85$ fold recognition tasks with positive or negative label, and protein sequences of varying lengths.

We also created two synthetic test sets for the scalability analysis. One of them contains $5$ sequences of length $128$. We used this test set to evaluate the effects of the hyperparameters $\numanc$ and $\kmer$ on the execution time of \fw{}. The second test set consists of $5$ sequences of different lengths to evaluate the effect of sequence length on execution time.

\subsection{Experimental Setup}
To conduct experiments, we used Amazon EC2 t2.xlarge instances. For the LAN setting, we selected all instances from the Frankfurt region, and for the WAN setting, we selected additional instances from London and Paris. We represent numbers with $64$ bits whose least significant $20$ bits are used for the fractional part. When computing the final mapping of the sequence, we chose to use $\mathsf{INVSQRT}$ to compute the inverse square root of the Gram matrices. Since the runtime of the experiments using the outsourced inverse square root of the Gram matrices is significantly shorter, we wanted to provide the worst case runtime results by using $\mathsf{INVSQRT}$.

\subsection{Correctness Analysis}
We conducted the experiments to analyze the correctness of the \app{} on the LAN. We performed the predictions of test samples of a task of SCOP to verify that we can get the same predictions as with RKN in plaintext. We selected the first task and trained a different RKN model for each combination of the parameters $\numanc \in \{16, 32, 64, 128\}$ and $\kmer \in \{5, 7, 10\}$. We randomly selected $5$ sequences from the test samples utilized by the RKN in this task and input these sequences into \app{}. Then, \app{} privately performed the prediction of these test samples on the secret shared model. At the end of this process, we received the shares of the predictions from the proxies and performed the reconstruction to obtain these predictions in plaintext. When we compared the predictions of \app{} with the predictions of plaintext RKN for these randomly selected test samples, the largest absolute difference between the corresponding predictions is less than $2 \times 10^{-5}$. Such close predictions of \app{} suggest that \fw{} can correctly perform almost the same predictions as RKN without sacrificing the privacy of the test samples or the model.  

There are two reasons for the small deviation in prediction results. The first is the truncation operation at the end of the multiplication to preserve the format of the resulting number \cite{mohassel2017secureml}. Although this causes only a small error for a single multiplication, it can add up for thousands of multiplications. Second, the utilized number format has inherently limited precision and this could cause a small error between the representation of the value in our format and the actual value itself. Such an error could also accumulate during the calculation and lead to a relatively high error in the end. Depending on the problem, one can address these problems by setting the number of bits for the fractional part to a higher number to increase the precision and minimize the overall error.

\begin{figure}[ht]
    \centering
    \begin{subfigure}[b]{0.232\textwidth}
         \centering
         \includegraphics[width=\textwidth]{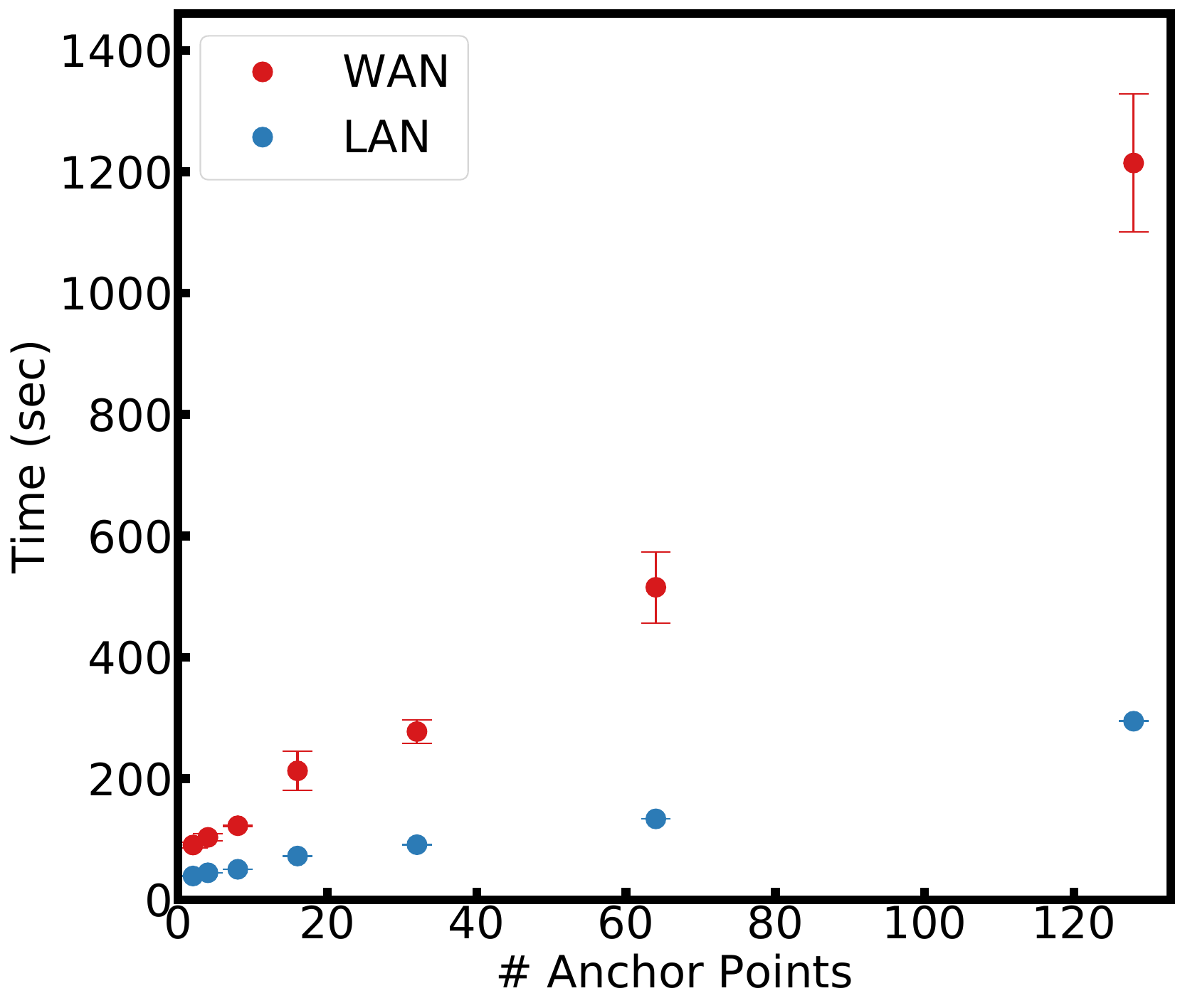}
         \caption{}
         \label{fig:anc}
     \end{subfigure} 
    \begin{subfigure}[b]{0.227\textwidth}
         \centering
         \includegraphics[width=\textwidth]{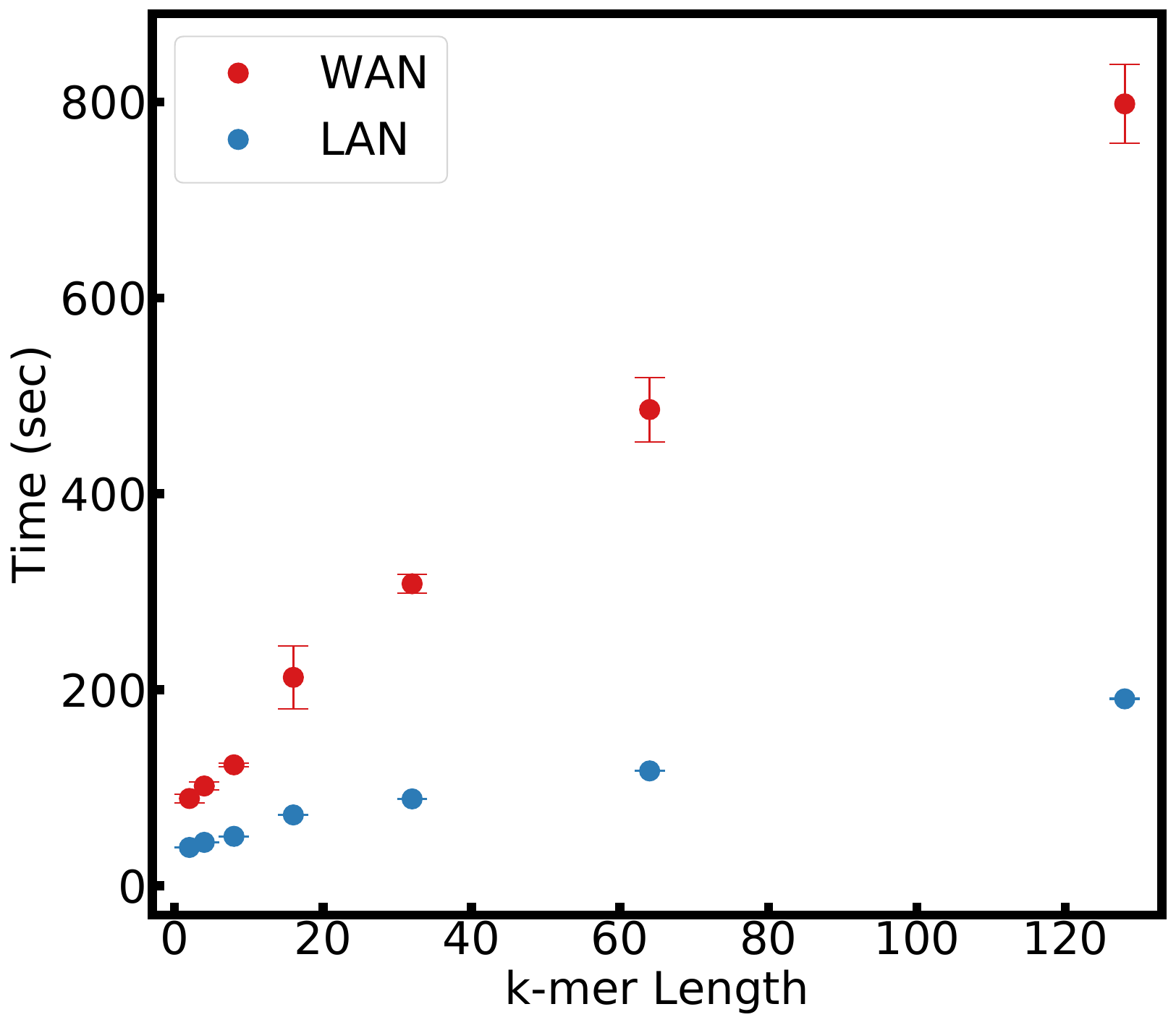}
         \caption{}
         \label{fig:kmer}
     \end{subfigure} \\
    \begin{subfigure}[b]{0.228\textwidth}
         \centering
         \includegraphics[width=\textwidth]{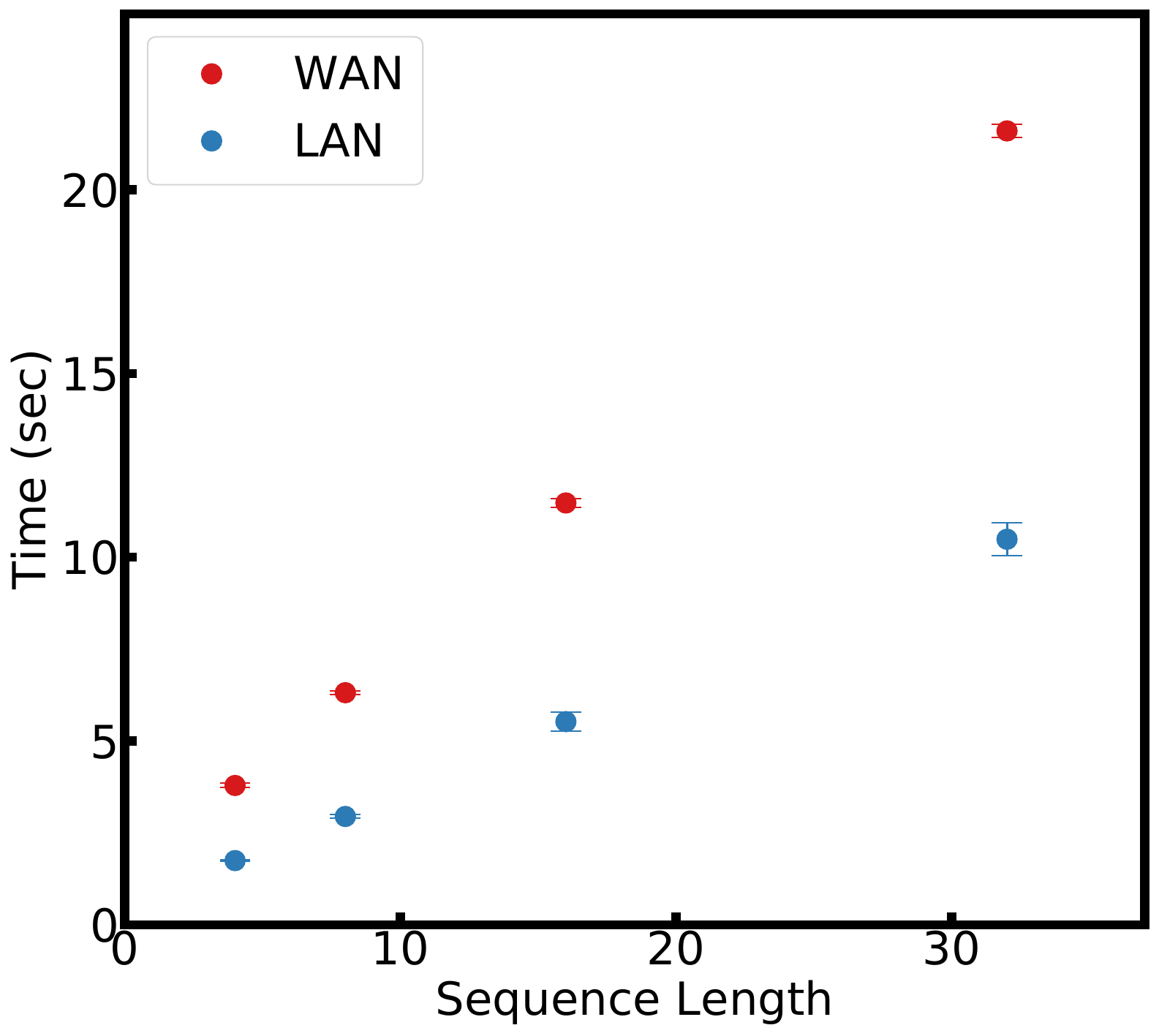}
         \caption{}
         \label{fig:seq}
     \end{subfigure}
    \caption{The results of the execution time analysis of \fw{} on both WAN and LAN settings for varying \textbf{(a)} number of anchor points for a fixed k-mer length and sequence length, \textbf{(b)} length of k-mers for a fixed number of anchor points and sequence length, and \textbf{(c)} length of sequences for a fixed number of anchor points and length of k-mer.}
    \label{fig:results}
\end{figure}

\subsection{Execution Time Analysis}
We also examined the effects of the parameters of the RKN, namely the number of anchor points, the length of the k-mers and the length of the sequence, on the execution time of \fw{} on both LAN and WAN. For this purpose, we used two synthetic datasets. In the analysis of a parameter, we vary that parameters and fixed the rest. Figure \ref{fig:results} summarizes the results and demonstrates the linear trend in the execution time of \fw{} for different parameters.


\section{Conclusion}
In this study, we propose a new comprehensive secure framework, called \fw{}, based on three party computation to enable privacy preserving computation of complex functions utilized in machine learning algorithms. Some of these functions have not been addressed before such as the exact exponential computation and the inverse square root of a Gram matrix and some of them had been solved either inefficiently or using a different methodology. We used \fw{} to realize the privacy preserving inference on a pre-trained RKN. To the best of our knowledge, \fw{} is the first study performing such an inference on RKNs. We demonstrate the correctness of \fw{} on the LAN using the tasks of the original RKN paper. The experiments demonstrate that \fw{} can give almost the exact same prediction result as one could obtain from the RKN without privacy. This demonstrates the correctness of the proposed methods. Moreover, we evaluated the performance of \fw{} in terms of execution time on two synthetic datasets. We carried out these analyses on both WAN and LAN, and showed that \fw{} scales linearly with the length of the k-mers and the length of the input sequence, and almost linearly with the number of anchor points. Note that \fw{} is not just limited to the privacy preserving inference on RKNs but also capable of realizing other machine learning algorithms requiring the provided building blocks. In future work, building on \fw{}, we plan to address the privacy preserving training of the RKN and other state-of-the-art machine learning algorithms. We also aim to improve the security of $\mathsf{INVSQRT}$ so that we do not lose the complete randomness of MPC in the computation of the inverse square root of a secret shared Gram matrix.

\bibliography{main}


\section{Supplement}
\subsection{Number Format}

To illustrate the number format, let $\numbit$ be the number of bits to represent numbers, $\dec$ be the number of bits allocated for the fractional part and $\mathbb{\setnf}$ be the set of values that can be represented in this number format, one can convert $x \in \mathbb{R}$ to $\repnf{x} \in \mathbb{\setnf}$ as follows:
\begin{equation}
    \repnf{x} = \left\{
        \begin{array}{ll}
            \floor{x * 2^\dec} & \quad x \geq 0 \\
            2^\numbit - \floor{|x * 2^\dec|} & \quad x < 0
        \end{array}
    \right.
\end{equation}
For example, $x = 3.42$ is represented as $\repnf{x} = 112066$, which, if you omit the leading zeros, is $1101101011100001010001$ in binary. The two most significant bits, the $21$st and $22$nd bits, are used to represent $3$ and the remainder represents $0.42$.

The choice of $\dec$ depends on the task for which \fw{} is used. If the numbers in the task are mostly small or they need to be as accurate as possible, the precision of the representation of the numbers is crucial. In such a case, a higher value for $\dec$ is required to allow for more decimal places of the original value, which in turn requires sacrifices in the upper limit that can be calculated by the building blocks of \fw{}. However, if the numbers that appear during the process are large, $\dec$ must be set to lower values to allow the integer part to represent larger numbers with fewer decimal places.

In our experiments with \fw{}, we represent the numbers with $64$ bits and set $\dec$ to $20$ to have a relatively high precision in representing the numbers and a sufficient range for the exponential. Although we have used the $64$ bit representation in the explanations of \fw{}'s functions, \fw{} can also work with a different number of bit representations. The reason for our choice is to take advantage of the natural modular operation of most modern CPUs, which operate with $64$ bits. In this way, we avoid performing modular operations after each arithmetic operation performed on the shares. However, one can still use a larger or smaller number of bits on $64$ bit CPUs by performing the modular operations manually.

\subsection{Discussion on MSB}
We give the detailed pseudocode of $\mathsf{MSB}$ in Algorithm \ref{alg:msb} in the main paper. Since it is optimized to reduce the communication round by integrating $\mathsf{MOC}$, in this section we give the compact version of the pseudocode of it in Algorithm \ref{alg:compact_msb} in the Supplement. We also give the security analysis of $\mathsf{MSB}$ based on this compact version to make it easy to follow.

\begin{algorithm}[!htb]
\footnotesize
\DontPrintSemicolon
\SetKwInOut{Input}{input}\SetKwInOut{Output}{output}
\SetKwFunction{algo}{$\mathsf{MSB}$}
\SetKwProg{myalg}{Algorithm}{}{}
\myalg{\algo{}}{
\Input{$\party_0$ and $\party_1$ hold ($\langle x\rangle_0$) and ($\langle x\rangle_1$), respectively}   
\Output{$\party_0$ and $\party_1$ get $\langle z\rangle_0$ and $\langle z\rangle_1$, respectively, where $z$ is equal to $0$ if the most significant bit of $x$ is $0$ and $1$ otherwise.}
$\party_0$ and $\party_1$ hold a common random bit $f$\;
For each $i \in \{0, 1\}$, $\party_i$ executes Steps $4$-$9$.\;
$\langle d\rangle_i^\ringB = \langle x\rangle_i \mod \ringB$\;
$\langle d\rangle_i=\mathsf{MOC}(\langle d\rangle_i^\ringB)$.\;
$\langle z\rangle_i = \langle x\rangle_i - \langle d\rangle_i$.\;
$\langle a[0]\rangle_i = if\ringB-\langle z\rangle_i$\; 
$\langle a[1]\rangle_i = i(1-f)\ringB-\langle z\rangle_i$\; 
$\party_i$ sends $\langle a\rangle_i$ to $\party_2$\;
$\party_2$ reconstructs $a[j]$ where $j \in \{0, 1\}$ and computes $a[j]=a[j]/\ringB$\;
$\party_2$ creates two fresh shares of $a[j]$ where $j \in \{0, 1\}$ sends them to $\party_0$ and $\party_1$\;
For each $i \in \{0, 1\}$, $\party_i$ executes Step $13$\;
$\langle z\rangle_i = \langle a[f]\rangle_i$\;
}
\captionsetup{width=\linewidth}
\caption{Compact and unoptimized version of Most Significant Bit ($\mathsf{MSB}$)}
\label{alg:compact_msb}
\end{algorithm}

\subsection{$\mathsf{MUL}$ and Details of Truncation}
Let $i \in \{0,1\}$ be and $\langle c\rangle_i = \langle a\rangle_i \cdot \langle b\rangle_i $ be the shares of the multiplication triple and $\langle x\rangle_i$ and $\langle y\rangle_i$ be the shares of $x$ and $y$, respectively, in $\party_i$. To compute $\langle z\rangle = \langle x\rangle \cdot \langle y\rangle $, $\party_i$ first computes $\langle e\rangle_i = \langle x\rangle_i - \langle a\rangle_i $ and $\langle f\rangle_i = \langle y\rangle_i - \langle b\rangle_i $, and then sends them to $\party_{1-i}$. Afterwards, $\party_i$ reconstructs $e$ and $f$ and calculates $\langle z\rangle_i = i\cdot e \cdot f + f \cdot \langle a\rangle_i + e \cdot \langle b\rangle_i + \langle c\rangle_i$. In the last step to obtain the shares of the multiplication, $\party_0$ shifts $z_0$ to the right by $f$. Slightly differently, $\party_1$ first shifts $(-1 * z_1)$ to the right by $f$ and then multiplies the result by $-1$.


\subsection{Details of Security Analysis}
In this part, we give the security analysis of $\mathsf{MOC}$, $\mathsf{MSB}$, $\mathsf{CMP}$, $\mathsf{MUX}$, $\mathsf{EXP}$ and $\mathsf{INVSQRT}$. In the security analysis of some of these methods, we take the security of the private compare functionality $\mathcal{F}_{\mathsf{PC}}$ \citep{wagh2019securenn} as base. We kindly refer the readers to \citet{wagh2019securenn} for the security analysis of $\mathcal{F}_{\mathsf{PC}}$.

\begin{lemma}
\label{lemma:mc}
The protocol $\mathsf{MOC}$ in Algorithm \ref{alg:mc} in the main paper securely realizes the functionality $\mathcal{F}_{MOC}$ in $\mathcal{F}_{\mathsf{PC}}$ hybrid model.  
\end{lemma}

\begin{proof}
First, we prove the correctness of our protocol by showing  $(\langle x\rangle_0^\ringB + \langle x\rangle_1^\ringB) \mod \ringB =  (\langle x\rangle_0 + \langle x\rangle_1) \mod \ringA$. In the protocol, $y = (x + r) \mod \ringB$ and $\mathsf{isWrap}(x,r,\ringB) = r \overset{?}{>} y$, that is, $\mathsf{isWrap}(x,r,\ringB)=1$ if $r > y$, $0$ otherwise. At the beginning, $\party_0, \party_1$ and $\party_2$ call $\mathcal{F}_{PC}$ to compute $c=r \overset{?}{>} y$ and $\party_0$ and $\party_1$ obtain the boolean shares $c_0$ and $c_1$, respectively. Besides, $\party_2$ sends also the boolean shares $w_0$ and $w_1$ of $w=\mathsf{isWrap}(\langle r\rangle_0,\langle r\rangle_1,\ringB)$ to $\party_0$ and $\party_1$, respectively. If $\mathsf{isWrap}(\langle y\rangle_0,\langle y\rangle_1,\ringB)$ is $1$ then $\party_0$ adds $\ringB$ to $\langle y\rangle_0$ to change the ring of $y$ from 
$\ringB$ to $\ringA$. To convert $r$ from ring $\ringB$ to ring $\ringA$, $\party_0$ and $\party_1$ add $\ringB$ to their shares of $r$ based on their boolean shares $w_0$ and $w_1$, respectively. If $w_0 = 1$, then $\party_0$ adds $\ringB$ to its $r_1$ and $\party_1$ does the similar with its shares. Later, we need to fix is the summation of $x$ and $r$, that is the value $y$. In case of $x+r \geq \ringB$, we cannot fix the summation value $y$ in ring $\ringA$ by simply converting it from ring $\ringB$ to ring $\ringA$. This summation should be $x+r$ in ring $\ringA$ rather than $(x+r)$ $mod$ $\ringB$. To handle this problem, $\party_0$ and $\party_1$ add $\ringB$ to their shares of $y$ based on their shares $c_0$ and $c_1$. As a result, we convert the values $y$ and $r$ to ring $\ringA$ and fix the value of $y$ if necessary. The final step to obtain $x_i$ for party $\party_i$ is simply subtract $r_i$ from $y_i$ where $i \in \{0,1\}$.

Next, we prove the security of our protocol. $\party_2$ involves this protocol in execution of $\mathcal{F}_{PC}$. We give the proof $\mathcal{F}_{PC}$ above. At the end of the execution of $\mathcal{F}_{PC}$, $\party_2$ learns $u^\prime$. However, $u^\prime = u \oplus (x>r)$ and $\party_2$ does not know $u$. Thus  $u^\prime$ is uniformly distributed and can be perfectly simulated with randomly generated values. $\party_i$ where $i \in \{0,1\}$ sees fresh shares of $\langle r\rangle_i^\ringB$, $\{\langle r[j]\rangle_{i}^{\ringC}\}_{j \in[\numbit]}$, $w_i^B$ and $u_i^B$. These values can be perfectly simulated with randomly generated values. 
\end{proof}


\begin{lemma}
\label{lemma:msb}
The protocol $\mathsf{MSB}$ in Algorithm \ref{alg:msb} in the main paper securely realizes the functionality $\mathcal{F}_{MSB}$ in $\mathcal{F}_{\mathsf{MOC}}$ hybrid model.  
\end{lemma}

\begin{proof}
First, we prove the correctness of our protocol. Assume that we have $\numbit$-bit number $u$. $v = u - (u \mod 2^{\numbit-1})$ is either $0$ or $2^{\numbit-1}$. In our protocol, $\langle z\rangle_i$ is the output of $\party_i$ where $i \in \{0,1\}$. We have to prove that $\mathsf{Reconstruct(\langle z\rangle_i)}$ is equal to the most significant bit (MSB) of $x$. $\party_i$ where $i \in \{0,1\}$ computes $d_i^\ringB=x_i \mod \ringB$ which is a share of $d$ over $\ringB$. $\party_i$ computes $d_i$ which is a share of $d$ over $\ringA$ by invoking $\mathsf{MOC}$. Note that $z= x-\mathsf{Reconstruct(\langle d\rangle_i)}$ and all bits of $z$ are $0$ except the MSB of $z$, which is equal to the MSB of $x$. Now, we have to map $z$ to $1$ if it is equal to $\ringB$ or $0$ if it is equal to $0$. $\party_0$ sends the $z_0$ and $z_0+\ringB$ in random order to $\party_2$ and $\party_1$ sends the $z_1$ to $\party_2$. $\party_2$ reconstructs two different values, divides these values by $\ringB$, creates two additive shares of them, and sends these shares to $\party_0$ and $\party_1$. Since $\party_0$ and $\party_1$ know the order of the real MSB value, they correctly select the shares of its mapped value.

Second, we prove the security of our protocol. $\party_i$ where $i \in \{0,1\}$ sees $\langle d\rangle_i$, which is a fresh share of $d$, and $\langle a[0]\rangle_i$ and $\langle a[1]\rangle_i$, one of which is a fresh share of the MSB of $x$ and the other is a fresh share of the complement of the MSB of $x$. Thus the view of $\party_i$ can be perfectly simulated with randomly generated values.
\end{proof}


\begin{lemma}
\label{lemma:cmp}
The protocol $\mathsf{CMP}$ in the main paper securely realizes the functionality $\mathcal{F}_{CMP}$ in $\mathcal{F}_{\mathsf{MSB}}$ hybrid model.  
\end{lemma}

\begin{proof}
First, we prove the correctness of our protocol. Assume that we have $x$ and $y$. We first compute $z = x - y$. If $z$ is negative, which corresponds to $1$ in the most significant bit of $z$, it means that $y > x$. In this case, $\mathsf{CMP}$ outputs $1$. If $z$ is non-negative, which corresponds to $0$ in the most significant bit of $z$, then it indicates that $x \geq y$. In this case, the output of $\mathsf{CMP}$ is $0$. Since the output of $\mathsf{CMP}$ exactly matches the output of $\mathsf{MSB}$ and we have already proved the correctness of $\mathsf{MSB}$, we can conclude that $\mathsf{CMP}$ works correctly.

Second, we prove the security of our protocol. Since $\langle z \rangle_i = \langle x \rangle_i - \langle y \rangle_i$ is computed locally by $\party_i$ for $i \in \{0,1\}$, it does not reveal any information about $x$ and $y$. Afterward, $\mathsf{MSB}$ is called on $\langle z \rangle_i$ to determine the most significant bit of $z$ in secret shared form. Considering that the security of $\mathsf{MSB}$ is proven, we can conclude that $\mathsf{CMP}$ compares two secret shared values without compromising their privacy.
\end{proof}

\begin{lemma}
\label{lemma:mux}
The protocol $\mathsf{MUX}$ in Algorithm \ref{alg:ss} in the main paper securely realizes the functionality $\mathcal{F}_{\mathsf{MUX}}$.  
\end{lemma}

\begin{proof}
We first prove the correctness of our protocol. $\langle z\rangle_i$ is the output of $\party_i$ where $i \in \{0,1\}$. We need to prove that $\mathsf{Reconstruct(\langle z\rangle_i)} = (1-b)x+by$. 

\begin{equation}
\scriptsize
\begin{split}
\langle z\rangle_0 + \langle z\rangle_1 ={}& \langle x\rangle_0-\langle b\rangle_0(\langle x\rangle_0-\langle y\rangle_0)+r_1\langle b\rangle_0 + r_2(\langle x\rangle_0-\langle y\rangle_0) \\
& + r_2r_3 + \langle x\rangle_1 - \langle b \rangle_1(\langle x\rangle_1-\langle y\rangle_1) + r_0(\langle x\rangle_1-\langle y\rangle_1) \\
& + r_0r_1 + r_3\langle b\rangle_1-\langle b\rangle_0\langle x\rangle_1+\langle b\rangle_0\langle y\rangle_1 - \langle b\rangle_0r_1 - r_0\langle x\rangle_1 \\
& + r_0\langle y\rangle_1-r_0r_1 - \langle x\rangle_0\langle b\rangle_1-\langle x\rangle_0r_2+\langle y\rangle_0\langle b\rangle_1 \\ 
& + \langle y\rangle_0r_2 - r_3\langle b\rangle_1-r_3r_2\\
={}& (1-\langle b\rangle_0-\langle b\rangle_1)(\langle x\rangle_0+\langle x\rangle_1) + (\langle b\rangle_0+\langle b\rangle_1)(\langle y\rangle_0+\langle y\rangle_1)\\
={}& (1-b)x+by
\end{split}
\label{eq:ss}
\end{equation}

Next we prove the security of our protocol. $\party_2$ gets $M_2,M_3,M_5$ and $M_6$. All these values are uniformly random values because they are generated using uniformly random values $r_0,r_1,r_2,r_4$. $\party_2$ computes $M_2M_5+M_3M_6$. The computed value is still uniformly random because it contains uniformly random values $r_0,r_1,r_2,r_4$. As a result, any value learned by $\party_2$ is perfectly simulated. For each $i \in \{0,1\}$, $\party_i$ learns a fresh share of the output. Thus $\party_i$ cannot associate the share of the output with the shares of the inputs and any value learned by $\party_i$ is perfectly simulatable.
\end{proof}



\begin{lemma}
\label{lem:exp}
The protocol $\mathsf{EXP}$ securely computes the exponential of a publicly known base raised to the power of a secret shared value.
\end{lemma}

\begin{proof}
We begin the proof by showing the correctness of the method. Let $x$ be the power whose representation in our number format is $\langle x \rangle$ and $b$ be the publicly known base. $\party_0$ or $\party_1$ computes $C_p = \{\ldots, b^{8}, b^{4}, b^{2}, b, b^{1/2}, b^{1/4}, b^{1/8}, \ldots\}$ and $C_n = \{\ldots, b^{-8}, b^{-4}, b^{-2}, b^{-1}, b^{-1/2}, b^{-1/4}, b^{-1/8}, \ldots\}$ and the other generates a corresponding set of $0$s for $C_p$ and $C_n$. These values in $C_p$ and $C_n$ correspond to $b^{2^i}$ and $b^{-1 \cdot 2^i}$, respectively, for $i \in \{(\numbit - \dec), \ldots, 2, 1, 0, -1, -2, \ldots, -\dec\}$ assuming that only the corresponding bit value of the power $x$ is $1$. They choose one of these sets based on the sign of $x$ and let $C$ be the selected set. Afterwards, they must choose between $c_j \in C$ and $1$ depending on $\langle x \rangle_j$ where $j \in \{0,1,\ldots,\numbit\}$. For this selection, they use the most significant bit operation on all cases where each bit of $x$ is at the most significant bit position. This is done by shifting the shares of $x$ to the left. Once they have the correct set of contributions, they basically multiply all of those contributions to obtain the result of the exponential. This proves the correctness of $\mathsf{ EXP }$.


\textbf{Corruption of a proxy:} At the beginning, since the adversary corrupting a proxy knows only one share of the power $x$, that is either $x_0$ or $x_1$, it cannot infer any information about the other share. The first step of the exponential is to compute the possible contribution of every bit of positive and negative power. This is publicly known. The following step is to select between these contribution depending on the result of $\mathsf{MSB}(x)$ by using $\mathsf{MUX}$. Since both $\mathsf{MSB}$ and $\mathsf{MUX}$ are secure, the adversary can neither infer anything about $x_{1-j}$ nor relate the share of the result it obtains to $x$ in general. In the next step, they obtain each bit of $x$ in secret shared form by using $\mathsf{MSB}$ and bit shifting on the shares of $x$. Considering the proven security of $\mathsf{MSB}$ and the shifting being simply local multiplication of each share by $2$, there is no information that the adversary could obtain. Afterwards, the proxies select the correct contributions by employing $\mathsf{MUX}$. Since $\mathsf{MUX}$ gives the fresh share of what is selected, the adversary cannot associate the inputs to the output. The last step is to multiply these selected contributions via $\mathsf{MUL}$, which is also proven to be secure. Therefore, we can conclude that $\mathsf{EXP}$ is secure against a semi-honest adversary corrupting a proxy.


\textbf{Corruption of the helper:} Since the task of the helper party in the computation of the exponential of a secret shared power is either to provide multiplication triples or to perform the required computation on the masked data, there is nothing that the adversary corrupting the helper party could learn about $x$. Therefore, it is fair to state that $\mathsf{EXP}$ is secure against a semi-honest adversary corrupting the helper.
\end{proof}


\begin{lemma}
\label{lem:invsqrt}
The protocol $\mathsf{INVSQRT}$ privately computes the inverse square root of a secret shared Gram matrix with a high probability.
\end{lemma}

\begin{proof}

We first demonstrate the correctness of $\mathsf{INVSQRT}$. Once the Gram matrix $\G$ whose eigenvalues and eigenvectors are $\evl$ and $\evc$, respectively, is masked by $\G^{'}= \M (\alp \G_i + \s I) \M^T$, we conduct eigenvalue decomposition on $\G^{'}$ resulting in the masked eigenvalues and eigenvectors of $\G$, which are $\evl^{'} = \alp \evl + \s$ and $\evc^{'} = \M \evc$, respectively. The recovery of $\evc$ can be achieved by $\evc = \M^T \evc^{'}$ since $M^T = M^{-1}$ due to the orthogonality. Since the masked eigenvalues $\evl^{'}$ are, once again, masked by computing $\evl^{''} = \Delta \evl^{'} + \alpha$, we have to compute the partial unmasker $U = \s \Delta + \alpha$ by using $\Delta$, $\alpha$ and $\s$, and subtract it from $\evl^{''}$ and divide the result by $\alp$ to obtain $\evl^{'''} = \Delta \evl$. Then, the multiplication of $(\evl^{'''})^{-1/2}$ and $\Delta^{1/2}$ gives us the reciprocal of the square root of the original eigenvalues, i.e. $\evl^{-1/2}$. Finally, we construct the inverse square root of $\G$ by computing $\evc \cdot diag(\evl^{-1/2}) \cdot \evc^{T}$. This concludes the correctness of $\mathsf{INVSQRT}$.


\textbf{Corruption of a proxy:} In case of a corruption of $\party_1$ by a semi-honest adversary, there is no gain for the adversary, since it only receives the share of the masked eigenvectors and the masks utilized by the helper party to mask the eigenvalues. At the end, $\party_1$ gets the share of the reciprocal of the square root of the eigenvalues as a secret shared result of $\mathsf{MUL}$. Therefore, $\mathsf{INVSQRT}$ is secure against an adversary corrupting $\party_1$. On the other hand, if an adversary corrupts $\party_0$, then it can reduce the possible eigenvalues of the Gram matrix into a specific set considering that $\M$, $\alp$, $\s$, $\Delta$ and $\alpha$ are from a specific range. Even in the worst case that the adversary deduces all the eigenvalues, which is extremely unlikely and negligible, it cannot reconstruct and learn the Gram matrix since $\party_0$ receives only the share of the eigenvectors, which is indistinguishable from a random matrix for $\party_0$. Thus, $\mathsf{INVSQRT}$ is also secure against an adversary corrupting $\party_0$.

\textbf{Corruption of the helper:} Since the values in $\M$, $\alp$ and $\s$ are from a specific range, the helper party can utilize the linear systems $\evl^{'} = \alp \evl + \s$ and $\evc' = M \evc$ to narrow down the possible values of $\evc$ and $\evl$, and $\G$ implicitly. Unless a masked value in the linear systems is at one of its edge cases, it is not possible for the adversary to infer the actual value. It is plausible to assume that not every value is at the edge case in these linear systems. Furthermore, considering that obtaining the eigenvalues of $\G$ has a relatively high chance, which is still extremely low, compared to obtaining the eigenvectors of $\G$, the adversary cannot form a linear system based on the inferred eigenvalues resulting in a unique solution for the eigenvectors, eventually a Gram matrix. Considering all these infeasibilities of obtaining the eigenvalues and eigenvectors in addition to the computational infeasibility and lack of validation method of a brute-force approach, the adversary cannot obtain $\G$ by utilizing $\evl'$ and $\evc'$. Therefore, $\mathsf{INVSQRT}$ is secure against a semi-honest adversary corrupting the helper party with a high probability.
\end{proof}

\begin{lemma}
\label{lem:pprkn_wo_invsqrt}
In case of the utilization of the outsourced inverse square root of the Gram matrices of the anchor points, the protocol \app{} securely realizes the inference of a test sample via the pre-trained RKN model. If $\mathsf{INVSQRT}$ is utilized, then the protocol \app{} privately performs the inference of a test sample via the pre-trained RKN model with a high probability.
\end{lemma}

\begin{proof}
In addition to the experimental proof of the correctness of \app{}, based on the hybrid model, we can state that \app{} correctly performs the private inference since it utilizes the methods of \fw{} which are already shown to perform their corresponding tasks correctly.

\textbf{Corruption of a proxy by a semi-honest adversary:} In case of utilization of the outsourced inverse square root of the Gram matrices, the private inference of RKN via \app{} involves a series of $\mathsf{MUL}$, $\mathsf{ADD}$, $\mathsf{EXP}$, $\mathsf{MSB}$, $\mathsf{MUX}$ and $\mathsf{CMP}$ operations on secret shared data, and addition and multiplication between secret shared data and a public scalar. Since we proved or gave the references to the proof of the security of those methods, we can conclude that \app{} is secure via the hybrid model. Instead of using the outsourced inverse square root of the Gram matrices, if we employ $\mathsf{INVSQRT}$ to compute those matrices, the adversary corrupting $\party_0$ can narrow down the set of possible eigenvalues of the matrices to a relatively small set of values. Even in the worst case, the adversary cannot obtain the Gram matrix since it has only the share of the eigenvectors of the Gram matrix. For an adversary compromising $\party_1$, there is no information gain at all. Therefore, \app{} performs the fully private inference in case of the utilization of the outsourced inverse square root of the Gram matrices and realizes the private inference with a high probability when $\mathsf{INVSQRT}$ is used.

\textbf{Corruption of the helper party by a semi-honest adversary:} In case of utilization of the outsourced inverse square root of the Gram matrices, considering the same reason we stated in the scenario where the adversary corrupts a proxy, we can conclude that \app{} is secure via the hybrid model. Instead of using the outsourced inverse square root of the Gram matrices, if we employ $\mathsf{INVSQRT}$ to compute those matrices, the adversary corrupting the helper party can narrow down the set of possible values of the eigenvalues and the eigenvectors of the Gram matrix to a relatively small set of values. As we discuss in the security analysis of $\mathsf{INVSQRT}$, such a reduction on the set is not enough for the adversary to reconstruct the Gram matrix with a high probability. Therefore, \app{} is either fully secure or secure with a high probability depending on whether $\mathsf{INVSQRT}$ is used or not, respectively.
\end{proof}

\subsection{Discussion on Range of Random Values in $\mathsf{INVSQRT}$}
As we discussed in the main paper, $\mathsf{INVSQRT}$ requires the utilized random values to be from a specific range so that the masked values do not overflow and lose its connection to the real numbers. Let $\setR_{\alp}, \setR_{\s}, \setR_{\M}, \setR_{\Delta}$ and $\setR_{\alpha}$ be the set of values from which we can select $\alp, \s, \M_{ij}, \Delta$ and $\alpha$ in a certain setting, respectively. And, let $\setR_{\evl}$ be the set of eigenvalues of $\G$. Let us also assume that all these sets contain positive values. Then, we have the following constraints on these random values:
\begin{itemize}
    \item $\mathbf{\evl}$: Since the anchor points used in the private inference are $L_2$ normalized vectors, the dot product of two $L_2$ normalized vectors cannot exceed $1$. Furthermore, the diagonal of the Gram matrix of these anchor points are all $1$. Therefore, considering the fact that $\sum_{i}^{\numanc} \evl_i = Tr(\G)$, we can confidently have the following constraint:
    \begin{equation}
        max(\setR_\evl) \leq \numanc
        \label{eq:const_evl}
    \end{equation} 
    
    \item $\mathbf{\M}$: In the masking of $\G$, we compute $\G^{'}= \M (\alp \G_i + \s I) \M^T$. Considering the boundary case where every entry of $\G$ is $1$, we can rewrite the computation of $\G_{kl}$ as follows:
    \begin{equation*}
        \G^{'}_{kl} = \sum_{u}^{q} \sum_{v}^{q} \M_{kv} \big( \delta(u,v) (\alp + \s) + (1 - \delta(u,v) \alp \big) \M_{lu}
        \label{eq:masked_g_entry}
    \end{equation*}
    In this computation, we have $\numanc^2$-many addition operations each of which contains $2$ multiplication operations. Assume that $\M_{ij}$ is the largest entry of $\M$, that is $\M_{ij} = max(\setR_{\M})$. At some point, we have to compute $z = \M_{ij} (\alp + \s) \M_{ij}$. Due to the multiplication and the number format that we used, $z$ has to be smaller than $2^{\numbit - 2 \dec - 1}$. Since $\G^{'}_{kl} \leq 2^{\numbit - \dec - 1}$ not to lose the connection to the real numbers, in the worst case that every entry of $\M$ equals to $\M_{ij}$, we have the constraint $\numanc^2 (\M_{ij} (\alp + \s) \M_{ij}) < 2^{\numbit - \dec - 1}$. Considering all these constraints, we have the following limitation on the maximum entry of $\M$:
    \begin{equation}
        max(\setR_\M)^2 < \begin{cases}
                        \frac{2^{\numbit - 2 \dec - 1}}{\alp + \s}  \vspace{0.2cm} \\ 
                        \frac{2^{\numbit - \dec - 1}}{\numanc^2 (\alp + \s)}  \\
                    \end{cases}
        \label{eq:const_orth}
    \end{equation}
    The lowest of these two constraints determines the boundary of the values of $\M$'s entries.
    
    \item $\mathbf{\Delta}$: In $\mathsf{INVSQRT}$, $\party_2$ masks the masked eigenvalues by computing $\evl^{''} = \Delta \odot \evl^{'} + \alpha$ or $\evl^{''} = \Delta \odot (\alp \evl + \s) + \alpha$ if we write $\evl^{'}$ more explicitly. In the computation of $\evl^{''}$, the result of an elementwise multiplication cannot be larger than $2^{\numbit - 2 \dec - 1}$ and the overall result cannot exceed $2^{\numbit - \dec - 1}$. For the range of $\Delta_i$ where $\Delta_i$ is the maximum element of $\Delta$, we have the following piecewise limitation:
    \begin{equation}
        max(\setR_\Delta) < \begin{cases}
                        \frac{2^{\numbit - 2 \dec - 1}}{\alp \numanc + \s}  \vspace{0.2cm} \\ 
                        \frac{2^{\numbit - \dec - 1} - \alpha}{\alp \numanc + \s}  \\
                    \end{cases}
        \label{eq:const_delta}
    \end{equation}
    
    \item $\mathbf{\alp}$: For the upper bound on $\alp$, let us rewrite Equations \ref{eq:const_orth} and \ref{eq:const_delta}:
    \begin{equation}
        max(\setR_\alp) < \begin{cases}
                        \frac{2^{\numbit - 2 \dec - 1}}{\M_{ij}^2} - \s  \vspace{0.2cm}\\ 
                        \frac{2^{\numbit - \dec - 1}}{\numanc^2 \M_{ij}^2} - \s  \vspace{0.2cm} \\
                        \frac{2^{\numbit - 2 \dec - 1} - \s \Delta}{\Delta \numanc}  \vspace{0.2cm} \\ 
                        \frac{2^{\numbit - \dec - 1} - \alpha - \s \Delta}{\Delta \numanc}  \\
                    \end{cases}
        \label{eq:const_alp}
    \end{equation}

    \item $\mathbf{\s}$: In the initial masking of the eigenvalues of $\G$, we compute $\alp \G_i + \s I$ ignoring the orthogonal matrix multiplications. This leads to the masking $\evl^{'} = \alp \evl + \s$. In this masking, we have limited range for the values. Due to its nature of such a limitation, we have edge cases which one can uniquely map to corresponding $\alp$, $\evl$ and $\s$ values. In order to minimize this, $\s$ has to be large enough to cover the total gap between the largest and the third largest possible result of the multiplication between the possible values of $\alp$ and $\evl$. Note that the largest possible gap is $min(max(\setR_\alp), max(\setR_\evl))$ and it is mostly fair to assume that $max(\setR_\alp) > max(\setR_\evl)$. Then, we can approximate the lower bound of $min(\setR_\s)$ as follows:
    \begin{equation*}
        min(\setR_\s) > 2 \cdot max(\setR_\evl)
    \end{equation*}
    When we set $max(\setR_\evl)$ to $\numanc$ for the worst case scenario, we end up with the following constraint:
    \begin{equation}
        min(\setR_\s) > 2 \numanc
        \label{eq:const_min_s}
    \end{equation}
    With the assumption that $\setR_{\alp}$ and $\setR_{\evl}$ contains positive values, this lower bound on $\s$ reduces the number of cases that can be uniquely mapped to $min(max(\setR_\alp), max(\setR_\evl)) + 1$.
    For the upper bound on $\s$, let us rewrite the constraints given in Equation \ref{eq:const_alp}, we obtain the following upper bound on $\s$:
    \begin{equation}
        max(\setR_\s) < \begin{cases}
                        \frac{2^{\numbit - 2 \dec - 1}}{\M_{ij}^2} - \alp  \vspace{0.2cm}\\ 
                        \frac{2^{\numbit - \dec - 1}}{\numanc^2 \M_{ij}^2} - \alp  \vspace{0.2cm} \\
                        \frac{2^{\numbit - 2 \dec - 1}}{\Delta} - \alp \numanc  \vspace{0.2cm} \\ 
                        \frac{2^{\numbit - \dec - 1} - \alpha}{\Delta} - \alp \numanc  \\
                    \end{cases}
        \label{eq:const_max_s}
    \end{equation}
    
    \item $\mathbf{\alpha}$: Considering the masking that we described in the discussion of $\Delta$ and the argument on the choice of lower bound of $\s$, we have the following lower and upper bounds for $\alpha$:
    \begin{equation}
    \begin{split}
        min(\setR_\alpha) > 2 \cdot min\big(max(\setR_\Delta), & max(\setR_\alp) \cdot max(\setR_\evl) \\
        & + max(\setR_\s)\big)
        \label{eq:const_min_alpha}
    \end{split}
    \end{equation}
    \begin{equation}
        max(\setR_\alpha) < 2^{\numbit - \dec - 1} - 2^{\numbit - 2 \dec - 1}
        \label{eq:const_max_alpha}
    \end{equation}
    
\end{itemize}

In order to demonstrate the resulting ranges, let us image a scenario where we set the upper limits of the random values as $max(\setR_\evl) = 2^5$, $max(\setR_\alp = 2^5)$, $max(\setR_\s = 2^{10})$, $max(\setR_{\M} = 2^5)$, $max(\setR_\Delta = 2^{10})$ and $max(\setR_\alpha = 2^{30})$. Let us also set $\dec = 20$ and $\numbit = 64$. Considering the number format that we use, the size of these set are $max(\setR_\evl) = 2^{25}$, $max(\setR_\alp = 2^{25})$, $max(\setR_\s = 2^{30})$, $max(\setR_{\M} = 2^{25})$, $max(\setR_\Delta = 2^{30})$ and $max(\setR_\alpha = 2^{50})$. 

These ranges are smaller than the original range of the values in \fw{}. As we also stated in the main paper, $\mathsf{INVSQRT}$ does not provide the full randomness of MPC. It sacrifices the randomness to some extent for the sake of the computability of the inverse square root of a secret shared Gram matrix. $\mathsf{INVSQRT}$ is an attempt towards such a computation via MPC without sacrificing the randomness of it at all. For the time being, one can choose the option of utilizing outsource inverse square root of Gram matrices if the full randomness is required. In that case, the complete randomness of MPC is valid for every operation performed in the privacy preserving inference on RKNs.

\end{document}

%% file: notation.tex
\usepackage{booktabs}       
\usepackage{natbib}
\usepackage{graphicx}
\usepackage[linesnumbered,vlined,boxed,commentsnumbered]{algorithm2e}
\usepackage{multirow}
\usepackage{colortbl}
\usepackage[usenames,dvipsnames,svgnames,table]{xcolor}
\usepackage{tabularx}
\usepackage{hhline}
\usepackage{amsfonts}
\usepackage{caption}
\usepackage{amsmath}
\usepackage{mathtools}
\usepackage{subcaption}
\usepackage{amsthm}
\newcolumntype{A}{>{\centering\columncolor{gray!5!white}\arraybackslash}X}
\newcolumntype{B}{>{\centering\columncolor{gray!10!white}\arraybackslash}X}
\newcolumntype{C}{>{\centering\columncolor{gray!20!white}\arraybackslash}p{1.7cm}}
\newcolumntype{D}{>{\centering\columncolor{gray!30!white}\arraybackslash}p{1.7cm}}
\newcolumntype{E}{>{\centering\columncolor{gray!20!white}\arraybackslash}X}

\DeclarePairedDelimiter\floor{\lfloor}{\rfloor}
\newcommand{\fw}{CECILIA}
\newcommand{\app}{ppRKN}
\newcommand{\dec}{f} 
\newcommand{\numbit}{n} 
\newcommand{\numanc}{q} 
\newcommand{\kmer}{k} 
\newcommand{\seq}{s} 
\newcommand{\setnf}{S} 
\newcommand{\repnf}[1]{\hat{#1}} 
\newcommand{\ringA}{L} 
\newcommand{\ringB}{K} 
\newcommand{\ringC}{V} 
\newcommand{\ringBool}{B} 
\newcommand{\party}{P}
\newcommand{\onehotdim}{d} 
\newcommand{\clas}{w}
\newcommand{\G}{G} 
\newcommand{\M}{\mathcal{M}} 
\newcommand{\s}{s} 
\newcommand{\alp}{\tau} 
\newcommand{\evl}{\Lambda} 
\newcommand{\evc}{Q} 
\newcommand{\setR}{\mathbb{E}}

\theoremstyle{plain}
\newtheorem{theorem}{Theorem}

\newtheorem{lemma}[theorem]{Lemma}

\theoremstyle{definition}

\theoremstyle{remark}